\documentclass[10pt, conference, compsocconf]{IEEEtran}
\IEEEoverridecommandlockouts
\usepackage{cite}
\usepackage{url}
\usepackage{graphicx}
\usepackage[outdir=./images/]{epstopdf}
\usepackage[caption=false, font=footnotesize]{subfig}
\usepackage{mathptmx}
\usepackage{amsmath}
\usepackage{amssymb}
\usepackage{amsthm}
\usepackage{algorithmic}
\usepackage{textcomp}
\usepackage{xcolor}
\def\BibTeX{{\rm B\kern-.05em{\sc i\kern-.025em b}\kern-.08em
    T\kern-.1667em\lower.7ex\hbox{E}\kern-.125emX}}

\usepackage{mathptmx}
\usepackage{amsthm}

\newtheorem{theorem}{Theorem}[section]
\newtheorem{lemma}[theorem]{Lemma}

\theoremstyle{definition}

\newcount\Comments  % 0 suppresses notes to selves in text
\newcommand{\kibitz}[2]{\ifnum\Comments=1{\textcolor{#1}{#2}}\fi}

\newcommand{\E}{\mathbb{E}}
\newcommand{\N}{\mathbb{N}}

\begin{document}

% \title{Phantom Steering in Stack Overflow}
% \title{A New model of Badge Behavior in Q\&A sites: Steering and Phantom Steering\\
\title{The Phantom Steering Effect in Q\&A Websites}
% \thanks{Stack Overflow}
% }
% \author {Nicholas Hoernle\and  Gregory Kehne\and Ariel D. Procaccia\and Kobi Gal
\author{
\IEEEauthorblockN{Nicholas Hoernle}
\IEEEauthorblockA{
\IEEEauthorrefmark{1}
University of Edinburgh\\
Edinburgh, UK\\
n.s.hoernle@sms.ed.ac.uk}
\and
\IEEEauthorblockN{Gregory Kehne, Ariel D. Procaccia}
\IEEEauthorblockA{
Harvard University\\
Massachusetts, USA\\
\{gkehne,arielpro\}@\{g,seas\}.harvard.edu}
\and
\IEEEauthorblockN{Kobi Gal\IEEEauthorrefmark{1}\IEEEauthorrefmark{2}}
\IEEEauthorblockA{
\IEEEauthorrefmark{2}
Ben Gurion University\\
Be'er Sheva, Israel\\
kobig@bgu.ac.il}
}

\maketitle

\begin{abstract}
Badges are commonly used in online platforms as incentives for promoting contributions. 
It is widely accepted that badges ``steer'' people?s behavior toward increasing their rate of contributions before obtaining the badge.
This paper provides a new probabilistic model of user behavior in the presence of badges. 
By applying the model to data from thousands of users on the Q\&A site Stack Overflow, we find that steering is not as widely applicable as was previously understood.
Rather, the majority of users remain apathetic toward badges, while still providing a substantial number of contributions to the site.
An interesting statistical phenomenon, termed ``Phantom Steering,'' accounts for the interaction data of these users and this may have contributed to some previous conclusions about steering.
Our results suggest that a small population, approximately $\mathbf{20\%}$, of users respond to the badge incentives.
Moreover, we conduct a qualitative survey of the users on Stack Overflow which provides further evidence that the insights from the model reflect the true behavior of the community.
We argue that while badges might contribute toward a suite of effective rewards in an online system, research into other aspects of reward systems such as Stack Overflow reputation points should become a focus of the community.
\end{abstract}

% Badges are commonly used in online platforms as incentives for promoting contributions. It is widely accepted that badges "steer" people?s behavior toward increasing their rate of contributions before obtaining the badge. This paper provides a new probabilistic model of user behavior in the presence of badges. By applying the model to data from thousands of users on the Q&A site Stack Overflow, we find that steering is not as widely applicable as was previously understood. Rather, the majority of users remain apathetic toward badges, while still providing a substantial number of contributions to the site. An interesting statistical phenomenon, termed "Phantom Steering," accounts for the interaction data of these users and this may have contributed to some previous conclusions about steering. Our results suggest that a small population, approximately 20%, of users respond to the badge incentives. Moreover, we conduct a qualitative survey of the users on Stack Overflow which provides further evidence that the insights from the model reflect the true behavior of the community. We argue that while badges might contribute toward a suite of effective rewards in an online system, research into other aspects of reward systems such as Stack Overflow reputation points should become a focus of the community.


\section{Introduction}
A well known finding from behavioral science research is that efforts towards a goal increase with proximity to the goal. 
This phenomenon, termed the goal-gradient hypothesis, has been demonstrated in a variety of settings, from animal studies in the lab to consumer purchasing behavior~\cite{hull1932goal,kivetz2006goal}.
More recently, the goal-gradient effect was observed in people's behavior in online communities that use virtual rewards, such as badges and reputation points, to increase users' contributions to the site~\cite{mutter2014behavioral,anderson2013steering}.
Recent examples of online settings that are using badges include communication platforms such as MS teams, ride-sharing platforms such as Lyft and online learning platforms such as Duolingo.  
We study this ``steering'' phenomenon in one such community, that of Stack Overflow (SO), where users can acquire badges for making different contributions to the platform, such as editing and voting on posts.
We identify \textit{who} exhibits steering, who does not, and \textit{how} this steering behavior can be characterised from observational data. 
Our surprising result is that a large population (at least $40\%$)  of SO users who are highly active, do not appear to exhibit steering towards badges. 

We present a generative model of steering which models users as having default activity rates which they can deviate from when approaching the requirements for achieving a badge.
The model is able to fit a complex multimodal distribution over the parameters that govern users' activities. 
This allows it to capture different levels of steering in the population.  
We applied the model to data collected from thousands of SO users, and investigate the following research questions: 
{Are all badge achievers affected by the steering (or goal-gradient) hypothesis in the same way? If some users do not steer, what portion of the population falls under this category? Does the presence of these users in the data set change any conclusions that were previously drawn about the phenomenon of steering?}

Our results revealed the following insights:  
First, more than 40\% of the users are not steered, in that they  exhibit a consistent activity rate in SO that is not effected by the badge.
We prove that a ``bump'' in activity conveyed by prior work arises as an artifact of centering the data on the day of badge achievement~\cite{anderson2013steering,yanovsky2019one}. 
We call this phenomenon \emph{the Phantom Steering Effect}.  
Second, about 20\% of users are steered, in that they dramatically increase their rate of activity prior to achieving the badge.
It is the effect that this small population of steered users has on aggregate measures that has led to the previous and broader claims of steering~\cite{anderson2013steering,yanovsky2019one,mutter2014behavioral}. 
Third, the majority of these steered users decrease their activity rate beyond what is claimed in prior work~\cite{anderson2013steering}, reaching close to $0$ after the badge has been achieved. 
Lastly, our conclusions are supported by responses to a user study that included $70$ active SO participants, in which only $24\%$ of participants selected badges as a motivating factor for their contribution.
 
Our study has important ramifications for system designers who invest resources into the implementation of badge rewards systems and for researchers who wish to understand the factors that contribute towards users' continued participation in online communities.

% This paper is first to provide a data-driven approach to the design and evaluation of models for studying user behavior in the presence of threshold badges. 
% The model was fit and  tested on real-world data (e.g., from SO in this paper) and can thereby be used to test hypotheses about users' steering behavior, leading to better understanding of how badges motivate users in online communities.

\section{Related Work}
\label{sec:related_work}
% \nh{Question to Kobi/Ariel.. The IEEE style guide asks that we use ``[1]'' to refer to the related work rather than ``Anderson et al., [1]''. Are you happy that I make this change? Unless it is the start of a sentence then we are supposed to say ``Reference [1]''.. I find this quite odd.}
We begin by relating to the general literature on the effect of badges in online communities. We then present in detail the specific work of \cite{anderson2013steering} which helps to motivate the generative models that we develop in Section~\ref{sec:modeling_user_contrbutions}.

\subsection{The Study of Online Badges}
\label{sec:general_badges}

The goal-gradient hypothesis stems from behavioral research where animals were observed to increase their effort as they approach a reward~\cite{hull1932goal, kivetz2006goal}. 
Kivetz et al.,\cite{kivetz2006goal} studied the behavior of different populations of people who were working toward various rewards.
They concluded that the goal-gradient hypothesis also holds true for people. 
Subjects who received a loyalty card, which tracked the number of coffees purchased from a local coffee chain, purchased coffee significantly more frequently the closer they were to earning a free cup of coffee. 
The authors recognized the existence of a group of participants who did not complete their coffee cards for the duration of the study, and did not exhibit a noticeable change in their coffee purchasing habits. 
They concluded that the loyalty card effect was constrained to the population of participants who handed in their completed loyalty cards in exchange for the free-coffee reward.

Anderson et al.,\cite{anderson2013steering} and Mutter and Kundisch\cite{mutter2014behavioral} were the first to study the goal-gradient hypothesis in online settings. 
They studied the \textit{observed} effect of badges on the behavior of participants in large Q\&A sites.
Both studies found evidence that users increase their rate of work as they approach the badge threshold.
However, they did not address the possibility that some users might achieve the badge as a consequence of their routine interactions on the website rather than being steered by the badge. 
% is analogous to the coffee-drinker who did not complete a loyalty card.
There is a possibility that people's actions are governed by motivations other than badges. 
%and they do not cha their activities in direct pursuit of a badge.
We extend these works by allowing for this possibility, such that we can characterize the true changes to users' behavior under the influence of a badge, and distinguish this from the case where users do not noticeably change their interaction behavior.

Other studies have independently confirmed that the presence of online badges increases the probability that a user will act in a manner to achieve the badge, as well as the rate at which the user will perform those actions~\cite{kusmierczyk2018causal,yanovsky2019one,bornfeld2017gamifying,ipeirotis2014quizz}. Kusmierczyk and Gomez-Rodriguez\cite{kusmierczyk2018causal} highlight the importance of modeling the ``utility heterogeneity'' among the users but they study badges which have a threshold of $1$ and do not characterize \textit{how} one might change one's behavior in the presence of the badge incentive. 
Yanovsky et al.,\cite{yanovsky2019one} study the presence of different populations within the SO database by employing a clustering routine. 
They discovered notably different responses to the badge based on the cluster that a user belongs to. 
Their study did not acknowledge the possibility that the observed data might be consistent with a hypothesis that some users do not exhibit steering.
Anderson et al.,\cite{anderson2014engaging} studied the implementation of a badge system in a massive open online course and they provide a prescriptive system for the design of badges such that there is a maximum effect on the population. 
Zhang et al.,\cite{zhang2019reading} suggest that SO create new badges to encourage users to integrate helpful comments into the accepted answers. 
They thereby present an example of how system designers might use a badge to encourage a desired behavior from their user base.
In contrast to this, we suggest that badges have a limited scope and work should be completed to understand other motivations that the users' have such that better and more effective rewards can be designed to motivate online communities.

\subsection{A Utility Model for Steering}
\label{sec:anderson_model}

Most relevant to our work is the paper of \cite{anderson2013steering} who present a parametric description of a user's utility when the user is steered by badges.
    % \footnote{Anderson et al.,\cite{anderson2013steering} coined the term ``steering'' which refers to the goal-gradient effect in the context of badges.} 
% This model is generative in that it describes the change to a user's distribution over actions under the assumption that the user tries to maximize some utility derived by earning new badges.
The model describes a user as having a preferred distribution from which actions are sampled. 
As users approach the required threshold for achieving a badge, they \textit{deviate} from their preferred distributions. 
The deviation from the preferred distribution is controlled by the {utility} gained by achieving the badge and the {cost} for deviating from the preferred distribution. 

We let $A^d_u$ refer to the distribution over the count of actions that a user $u$ takes on day $d$.
The user's utility is a function of  $A^d_u$ and it is the sum of three terms.%
    \footnote{Our notation differs slightly from that of \cite{anderson2013steering}. Anderson et al.,\cite{anderson2013steering} uses a parameter $\mathbf{x_a}$ to refer to a user's distribution over the next action. We rather use $A_u^d$ to denote the distribution over the count of actions on a particular day. The two are linked (the distribution over the next action influences the count of actions on a specific day), however, we choose to model directly the data that is available from SO rather than a quantity that we do not observe.}
The first term, $\sum_{b \in B} I_b V_b$, is the non-negative
value that a user derives from already-attained badge rewards (where $V_b$ is the assumed value of a badge and $I_b$ is the indicator that the user has attained badge $b$). 
The second term, $\theta\mathbb{E}_{A_u^d}[U_{u,d+1}(A_u^{d+1})]$, describes the user's expected future utility, discounted by $\theta$, when acting under the distribution $A_u^d$.
The final term, $g(A_u^d,P_u^d)$, is a cost function that penalises the user for deviating from the preferred distribution $P_u^d$ on that day. 
The cost $g$ represents  the unwillingness of the users to change their behavior, and it is in tension with the users' desire to achieve future badges.

The utility on day $d$ for user $u$ is then~\cite{anderson2013steering}:
\begin{equation*}
    \label{eq:anderson_modified}
    U_{u,d}(A_u^d) =  \sum\limits_{b \in B}I_b V_u^b + \theta \mathbb{E}_{A_u^d}[U_{u,d+1}(A_u^{d+1})] - g(A_u^d, P_u^d) 
\end{equation*}

It is important to note that the cost term $g$ is non-zero only when users deviate from their preferred distribution $P_u^d$. 
As such, this model assumes users deviate only to attain the value from the badge and only if that value outweighs the cost that is paid for deviating. 
This means that a deviation on the rate of actions which are incentivized by the badge must be an increase before the badge is achieved and cannot be an increase after the badge is achieved. We will make these same assumptions in the models presented in Section~\ref{sec:generative_model}.

This utility-based model presents a compelling description of how people respond to badges; however, it was not evaluated or tested by fitting it to specific data from SO.
Rather, predictions of the model were compared to aggregated data from SO and we show in Section~\ref{sec:studying_mean_action_count} that the aggregated analysis from these count data can lead to incorrect conclusions.
The lack of analysis on individual level predictions limits the credibility of the study as well as its practical value --- it is difficult to apply the utility-based model to the placement of badges without a means of determining the appropriate model parameters for a given community of contributors.

In this work we address the shortcomings of the utility-based approach by introducing a probabilistic model which allows us to use the vast literature on posterior inference in such models to assist with parameter estimation~\cite{blei2014build,rezende2015variational,kingma2013auto,kingma2016improved,ranganath2014black}.
The probabilistic model has two advantages over this prior work: (1) posterior distributions for latent parameters in the model can be learnt from real-world interaction data and (2) the model's fit to data can be used to test and update scientific hypotheses (for example, in this paper we propose and validate that while some users may steer in a similar way, there exist users who do not experience steering).

\section{Modeling User Activities}
\label{sec:modeling_user_contrbutions}

We model users' activities in SO as a distribution over their action counts.
The model aims to incorporate the major aspects of the utility model from \cite{anderson2013steering} but it frames the problem such that parameters can be estimated from data and the models can be tested on their fit to unseen user action data to allow for model comparison~\cite{box1962useful,blei2014build}. Moreover, the model allows for different users to experience different levels of steering.

% In Section~\ref{sec:generative_model}, we begin by providing a conceptual generative model of user activities in SO. 
% Sections \ref{sec:likelihood} and \ref{sec:link_to_rate_params} describe the specifics of the model parameters. 
% Finally, in Section~\ref{sec:inference} we detail how the latent parameters are represented by the model.

\subsection{A Generative Model of Steering}
\label{sec:generative_model}

Let $P_u$ be a latent parameter that controls the rate of activity for user $u$; this is the \textit{preferred distribution} of user $u$.   
$P_u$ induces a probability distribution over the action counts $A_u$ of user $u$.
Let $\beta$ denote the deviation of the user's activity from $P_u$ as a result of steering.  
The observed data for each user, $A_u$, consists of daily action counts for a predetermined number of weeks before and after achieving the badge. 
Thus, for $D$ days of interaction, $P_u$, $A_u$ and $\beta$ are all vectors of length $D$. 

Fig.~\ref{fig:graphical_model} presents three plausible models of user behavior in SO. 
White circles denote latent random variables and colored circles denote observed random variables; solid lines represent conditional dependence between the random variables.   
Model 0 (left) describes a non-steering scenario, in which the observed action counts, $A_u$, depend only on the user's preferred distribution, $P_u$. 
Model 1 (center) is a steering model in which a user deviates systematically from $P_u$ in a manner that is controlled by $\beta$.
As the values for $\beta$ increase (above 0), the users experience an increased activity rate (above their preferred distribution). 
Similarly, as $\beta$ decreases (below 0), the users experience a decreased activity rate.
Model 1 assumes that all users are steered in the same way.
Model 2 (right) relaxes this assumption by introducing a user-specific strength parameter $S_u \in (0,1)$ which mediates the effect of $\beta$ for user $u$. 
As $S_u$ decreases, the users deviate less from their base distributions. 
When $S_u$ approaches zero, the users' activity converges to that described by Model 0. 

The parameter $\beta$, that controls how a user responds to the badge, is a vector of length $D$ (each day relative to the date of badge achievement). 
To reflect the intuition developed by \cite{anderson2013steering} and explained in Section~\ref{sec:anderson_model}, we constrain $\beta$ to be non-negative before \textit{day $0$} --- the day when the user achieves the badge. Moreover, we constrain $\beta$ to be non-positive after this day to reflect the intuition that a user gains no further utility from the badge once it has been achieved (and thus does not work harder than his preferred distribution $P_u$). $\beta$ therefore implicitly includes the trade-off between the cost function $g$ and the badge utility $V$ that is discussed in Section~\ref{sec:anderson_model}.

\begin{figure}[t]
  \centering
  \includegraphics[width=.9\linewidth]{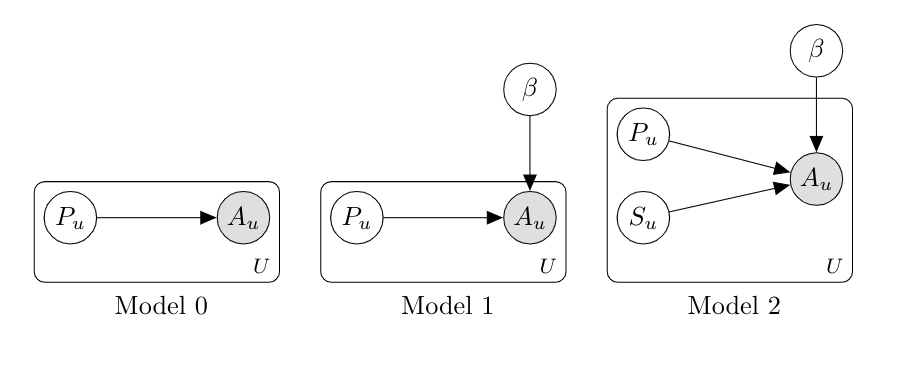}
  \caption{
  Model 0 (baseline model) has no notion of a badge -- only a user's preferred distribution induces the distribution over the observed actions. Model 1 allows for a global badge deviation ($\beta$) from a user's preferred distribution and it is experienced by all users. Model 2 has a user-specific strength parameter ($S_u$) that mediates the adherence to $\beta$.}
%   \Description{Steering graphical model}
  \label{fig:graphical_model}
\end{figure}

\subsection{Likelihood of Action Counts}
\label{sec:likelihood}

In this section we define the parameters that govern the distribution over users' action counts in SO. 
We wish to describe a variety of behaviors, including users who contribute sporadically and those who are more consistent. 
We therefore model action counts using a zero-inflated Poisson distribution.
The zero-inflated Poisson distribution has a rate parameter $\lambda_u^d$ and a Bernoulli probability $\alpha_u^d$ associated with each user $u$ and each day $d$ of interaction. 
The Bernoulli parameter $\alpha_u^d$ describes the probability that user $u$ is active or not on a given day $d$. 
The rate parameter $\lambda_u^d$ describes the expected count of actions that the user will perform under a Poisson distribution, conditioned on the user being active. 
Note that a user can be active on the platform without performing an action (e.g., logs on to the SO website but does not contribute). 
Conceptually, this would correspond to drawing a $1$ from the Bernoulli distribution but a count of $0$ actions from the Poisson distribution.

The probability that user $u$ performs $k$ actions on day $d$ is presented in \eqref{eq:likelihood_term}. 
We refer to the parameters $\alpha_u^d$ and $\lambda_u^d$ as a user's \textit{rate parameters} for day $d$.

% add explanation
% removed \mid \lambda_{u}^{d}, \alpha_{u}^{d} due to column width
\begin{equation}
    \label{eq:likelihood_term}
    Pr[A_{u}^{d} = k] = 
    \begin{cases}
        (1-\alpha_u^d) + \alpha_u^d Poisson(0 \mid \lambda_u^d),& \text{if } k = 0\\
        \alpha_u^d Poisson(k \mid \lambda_u^d),              & \text{otherwise}
    \end{cases}
\end{equation}

\subsection{Deriving the Rate Parameters $\alpha_u$ and $\lambda_u$}
\label{sec:link_to_rate_params}

This section connects the rate parameters   $\alpha_u$ and $\lambda_u$ to the generative models of Section~\ref{sec:generative_model}.
Each of $P_u$, $\beta$ and $S_u$ includes one component for $\alpha_u$ and one component for $\lambda_u$.
As such, for $D$ days of interaction, $P_u=({P_{u,1},P_{u,2}})$ comprises two real-valued vectors, each of length $D$. 
$P_{u,1}$ is the user's preferred distribution that is associated with $\alpha_u$ and $P_{u,2}$ is the user's preferred distribution associated with the parameter $\lambda_u$.
Similarly, $\beta=({ \beta_{1}, \beta_{2}})$ comprises two real-valued vectors of length $D$ that are associated with $\alpha_u$ and $\lambda_u$ respectively.
Finally, $S_u = (S_{u,1}, S_{u,2})$, is a tuple of two numbers between $0$ and $1$ which mediate (multiply) $\beta_1$ and $\beta_2$ for user $u$.

Equation \eqref{eq:generative_model1} 
derives a vector of probability values $\alpha_u$ (one for each day of interaction) as the element-wise sigmoid transformation of a vector that is the addition of the user's preferred distribution $P_{u,1}$ with $\beta_{1}$ where $\beta_{1}$ is mediated by (multiplied by) $S_{u,1}$.  Equation \eqref{eq:generative_model1} also derives a vector of strictly positive rate values $\lambda_u$ (one for each day of interaction) as the element-wise softplus transformation of the vector $P_{u,2} + S_{u,2} \times \beta_{2}$.

\begin{align}
\begin{split}
    \label{eq:generative_model1}
    \alpha_{u} &= \sigma \left( P_{u,1} + S_{u,1} \times \beta_1 \right)\\ 
    \lambda_{u} &= softplus \left( P_{u,2} + S_{u,2} \times \beta_2 \right)
\end{split}
\end{align}

\noindent The complete generative description for Model 2 is as follows (Models 1 and 0 are generated in the same way, with parameter $S_u$ set to 1 and 0  respectively):
\begin{enumerate}
    \item Sample $P_u$ and $S_u$ from their prior distributions (see Section~\ref{sec:inference}).
    \item Compute  $\alpha_u$ and $\lambda_u$ using \eqref{eq:generative_model1}.
    \item Sample the vector of the counts of actions for user $u$ from the zero-inflated Poisson as in \eqref{eq:likelihood_term}. 
\end{enumerate}

\subsection{Generating the Latent Parameters $P_u$ and $S_u$}
\label{sec:inference}
 
The distributions over $P_u$ and $S_u$ can be complex and multi-modal (reflecting the intuition that people display varying activity patterns). 
Following \cite{rezende2015variational}, we represent these distributions with a simple distribution (an $m$-dimensional isotropic Gaussian) which we transform via a series of bijective mappings to form a complex and possibly multi-modal distribution~\cite{rezende2015variational,papamakarios2019normalizing}.
The use of such transformations, called normalizing flows, has been shown to improve the modeling of complex distributions~\cite{papamakarios2019normalizing}.

%\nh{TODO: Please let me know if this next paragraph makes sense?}
The output of the normalizing flows corresponds to a sample in latent space from the generating distribution for a user's preferred distribution. 
We transform the output through a feed-forward network to a real-valued vector which is $P_u$. 
In this work we constrain $P_u$ such that it can learn a different distribution for each day of the week.
This enforces that $P_u$ does not evolve with a certain trend but it does allow different days to have different expected activity rates (e.g., some users might work more during the week than on the weekend).

The distribution over the steering parameters, $S_u$, for Model $2$ are assigned to be two elements in the $m$-dimensional vector after the transformations from the normalizing flows.
In this way, the latent dimensionality for all three models in Fig.~\ref{fig:graphical_model} is kept constant.
These real values are transformed via a sigmoid function to create $S_{u,1}, S_{u,2} \in (0,1)$.

\section{Amortized Variational Inference for Steering}
\label{sec:amortized_inference}

%\nh{notation in this section is that $p$ refers to a density. $Pr$ used in eq 1 and in Greg's proof reflects a probability. These are both standard.}

To infer the underlying parameters in the latent space, we use amortized inference~\cite{ranganath2014black,kingma2013auto}. 
Amortized inference uses a neural network to encode a data point into the latent parameters that are associated with its approximate posterior distribution.
Moreover, the inference objective allows model comparison such that hypotheses about the data can be tested (e.g., allowing us to validate the inclusion of the steering parameter $S_u$).

A fully-specified generative model defines a joint distribution over some latent random variables ($\mathbf{z}$) and the observed random variables ($\mathbf{x}$). 
The challenge is to infer the posterior of the latent parameters given the data that was actually observed $p(\mathbf{z} \mid \mathbf{x})$. 
For all but a handful of conjugate models, the posterior is intractable to derive analytically\cite{neal1993probabilistic,blei2003latent,hoffman2013stochastic}. 
% It is therefore common to use approximate methods which include Markov chain Monte Carlo~\cite{neal1993probabilistic} and variational inference~\cite{blei2003latent,hoffman2013stochastic}. 
Variational inference is a popular method for approximating the intractable posterior distribution by introducing a different (and more easily sampled from and evaluated) distribution over the same latent variables, $q(z)$.
By minimizing the KL-divergence between $q(z)$ and the true posterior $p(\mathbf{z} \mid \mathbf{x})$, one obtains an approximation to the true posterior~\cite{hoffman2013stochastic}.

% Traditionally, variational inference proposes to update the parameters of the approximation with a (optionally stochastic) coordinate ascent routine, presenting an algorithm that has strong ties to expectation-maximization~\cite{bishop2006pattern}. 
% More recently, \cite{kingma2013auto} proposed that the parameters of the approximation $q(z)$ rather be learnt as a function of an \textit{inference network} such that the shared parameters of the network \textit{amortize} the learning across the data examples.
% This new form of variational inference proposes that the parameters of the approximating distribution then are the transformation of the input data point through some network.
% The use of inference networks, along with the reparameterization trick (to allow for efficient back-propagation computation), has collectively been called ``black-box variational inference''~\cite{ranganath2014black}.
% If this \textit{inference network} is paired with a twin \textit{generator network}, the variational auto-encoder from \cite{kingma2013auto} is recovered.

It is important to note that minimizing the KL-divergence between $q(z)$ and $p(\mathbf{z} \mid \mathbf{x})$ is equivalent to maximizing the variational objective, called the Evidence Lower BOund (see \cite{hoffman2013stochastic} for a derivation and discussion). This \textit{ELBO} derives its namesake from the fact that it lower-bounds the marginal log-likelihood of the data under the assumptions of the model, a fact easily derived in the next equation, where Jensen's inequality is applied in the final line. 
It is due to this lower bound on the marginal log-likelihood, that it is also common to use the ELBO for model comparison (as is done in Section~\ref{sec:model_comparison}).

\vspace{-0.3in}
\begin{align*}
    %\label{eq:compute_elbo}
    \log p (x) &= \log \int p (x, z) dz \\
    &= \log \int p (x, z) \frac{q (z)}{q (z)} dz \\
    &= \log\mathbb{E}_{q_{}(z)} \left[ \frac{p (x, z)}{q (z)} \right] \\
    &\geq \mathbb{E}_{q_{}(z)} \left[ \log \frac{p (x, z)}{q (z)} \right] := ELBO(x)
\end{align*}

\section{Empirical Study}
\label{sec:emp_study}

SO has supplied us with the anonymized interaction data of users on the site from  January 2017 to April 2019. The data consist of a count of actions per user per day.
We focus our analysis on the $6,916$ users who achieved the  Electorate badge, which is awarded to users who vote on 600 questions, with at least 25\% of their  total  votes cast on questions.%
    \footnote{The users' voting data is not publicly available but qualitatively similar results can be obtained on other badge types in the freely available repository of data found at \url{https://archive.org/details/stackexchange}.}
This is the same badge type that was studied by \cite{anderson2013steering}.
The observed data is the number of  actions (question votes) per user per day for $5$ weeks before and after achieving the badge, making $70$ days of interaction per person.
Voting is a particularly important action to study as there are no reputation points that are awarded for voting on SO. 
Therefore, while we can obtain similar results on other badge types with data that are freely available,%
    \footnote{The modeling and inference code, details of network structure, and the application to the freely available ``editing'' data can be found at the repository: \url{https://github.com/NickHoernle/icdm_2020}. Moreover, see Appendix~\ref{sec_appendix:civic_duty} for a reproducible analysis on freely available data.} 
we focus on the voting badge here to reduce the possibility of the presence of the reputation points system confounding our analysis.

We compare the performance of Model $2$, Model $1$, and Model $0$ as described in~Section~\ref{sec:generative_model}. 
We also include a na\"ive baseline that uses maximum likelihood assignments for the rate parameters by setting $\alpha_u$ to equal the fraction of users who were active on day $u$, and $\lambda_u$ to equal the mean of the active users' action counts for day $u$.
For all models, we report two measures of performance: the evidence lower bound (the ELBO), which is the lower-bound on the marginal log-likelihood of the data under the model assumptions~\cite{kingma2013auto, hoffman2013stochastic, rezende2015variational}; and the mean square error (MSE) of the model for reconstructing the original number of actions for each user. 
Parameter estimation is done in Pytorch and Adam is used to maximize the ELBO  with a learning rate of 0.001.
We set the dimensionality of the latent space to $m:=20$ and use planar normalizing flows with $12$ layers~\cite{rezende2015variational}.

\subsection{Model Comparison}
\label{sec:model_comparison}

\begin{figure*}
  \centering
  \includegraphics[width=.85\textwidth]{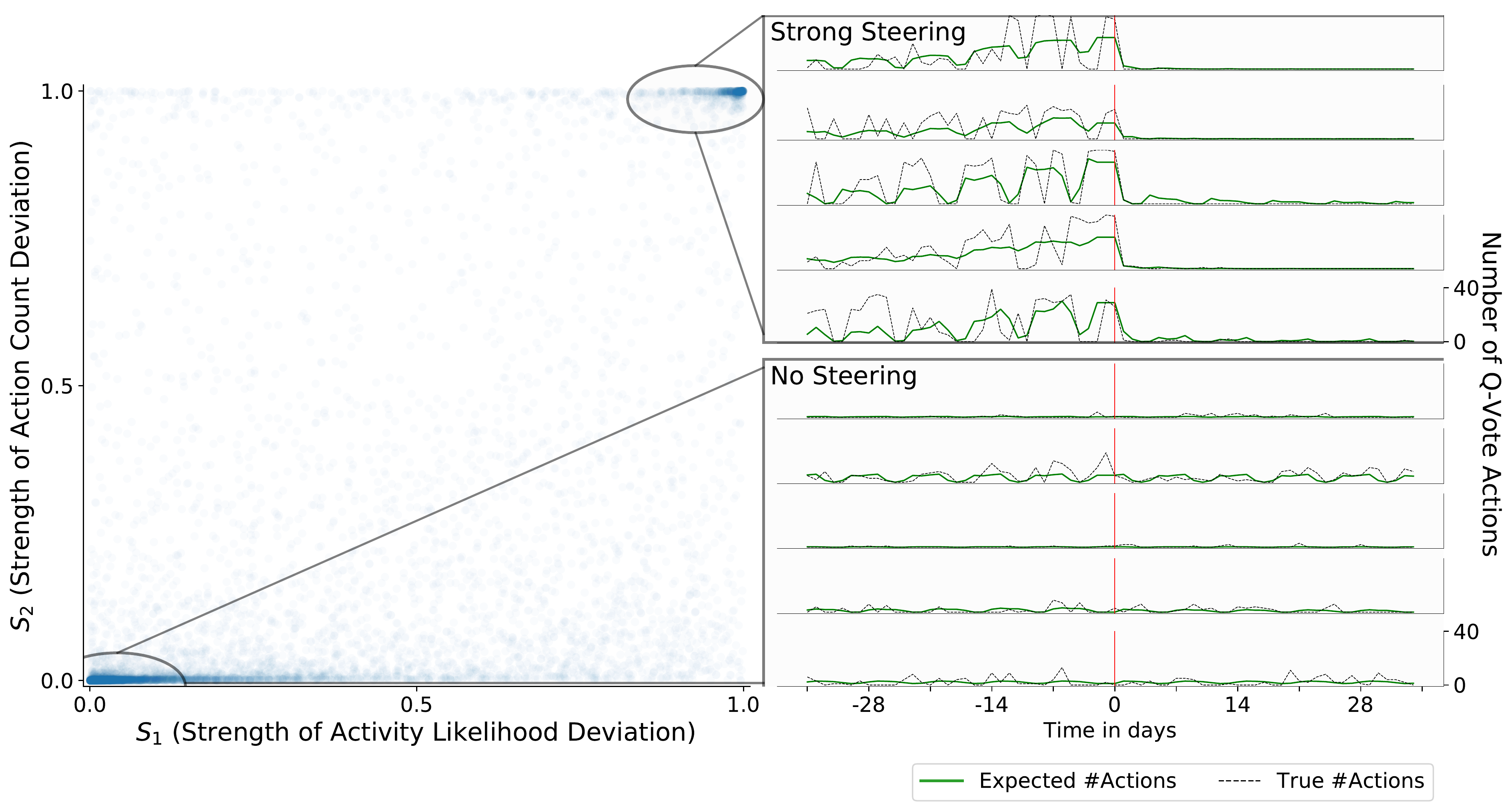}
  \caption{\textit{Left}: Scatter plot of inferred strength of steering measured by the likelihood that a user is active ($x$ axis) and the number of actions the user performs ($y$ axis). Values near $\mathbf{0}$ indicate little to no steering whereas values near $\mathbf{1}$ indicate users strongly steer toward achieving the badge. \textit{Top right}: $5$ samples of the action count trajectories from ``strong-steered'' users. The plots show the true number of question-vote actions (dashed line) alongside the expected number under the model (solid line) around day $0$ (the day of badge acquisition) marked by the red vertical line. These users dramatically increase their activity (number of actions) as they near day $0$ and sharply drop their activity after achieving the badge. \textit{Bottom Right}: $5$ samples of the action count trajectories from ``non-steered'' users. These users do not change their activity around day $0$. 
  }
%   \Description{Teaser}
  \label{fig:teaser}
\end{figure*}

We trained the models using the data of $4,149$ users (validation set of $1,386$ users,) while results are reported on a hold-out test set of $1,381$ users. 
Table~\ref{tab:elbo_results} compares the performance of the models on this test set. 
The results from Table~\ref{tab:elbo_results} show that Model $2$ outperforms the other models achieving a higher bound on the marginal log-likelihood and a lower reconstruction error on unseen data.
Moreover, the benefit of the amortized approach (which learns a complex representation of $P_u$ for each user) is demonstrated in that Models $0$, $1$ and $2$ all had a higher ELBO and consequently a lower reconstruction error than the na\"ive baseline. 
Both Model $1$ and $2$ outperform Model $0$, suggesting that the inclusion of the steering parameter $\beta$ does increase the probability of the data.

\begin{table}[b]
  \caption{ELBO (lower bound to log-likelihood) and mean squared error reconstruction (MSE).}
   \label{tab:elbo_results}
   \begin{center}
    \begin{tabular}{l|cc}
        \hline
        Model & ELBO &  MSE\\
        \hline
        Model 2 (w/ $S_u$)          & $\mathbf{-79.3}$  & $\mathbf{0.126}$   \\
        Model 1 (w/o $S_u$)         & $-104.9$           & 0.174             \\
        Model 0 (Baseline)          & $-116.1$           &  0.179             \\
        Na\"ive Baseline            & $-460.7$           & 49.33             \\
        \hline
    \end{tabular}
   \end{center}
\end{table}

Fig.~\ref{fig:teaser} (left) presents a scatter plot of the magnitude of the inferred $S_{u,1}$ and $S_{u,2}$ parameters. 
Each point in the scatter plot corresponds to a user. 
The scatter plot shows two clear modes.
One mode is in the top right corner of the plot which corresponds to users with $(S_{u,1}, S_{u,2}) \approx (1,1)$. 
We refer to the roughly $20\%$ of users in this mode as the ``strong-steerers'' because they adhere strictly to the $\beta$ deviation. 
The second, and larger, mode is in the bottom left of the plot which corresponds to $(S_{u,1}, S_{u,2}) \approx (0,0)$. 
We refer to the roughly $40\%$ of users in this mode as the ``non-steerers'' as they appear to eschew the $\beta$ deviation. 

We show   $5$ samples from the strong-steered population on the top-right hand corner of  Fig.~\ref{fig:teaser}.
The plots show the true count of actions as a function of time alongside the expected number of actions under the assumptions of Model $2$.
The red vertical line, on day $0$, corresponds to the day that the user achieved the Electorate badge.
It is important to note the high number of actions (both expected and true) before day $0$ when the badge was achieved.
After day $0$, both the true and expected numbers of actions drops dramatically.

In contrast to these strong-steered users, the bottom right of Fig.~\ref{fig:teaser} presents $5$ samples from the non-steered population. 
The counts of actions appear to show no change around day $0$.
{These users appear not to change their behavior in the presence of the badge}.

\subsection{Analysis of Steering}

The form of the inferred parameter $\beta$ shows the  effect of steering on users over time.
The plot of $\beta$ as a function of time is presented in Fig.~\ref{fig:steer_beta_response}. 
The magnitude of the values of $\beta$  indicate direct changes to the probability that the user is active, as well as expected changes in the number of actions on a given day. 
In accordance with related work, users increase their rate of activity as they approach the day upon which they achieved the badge~\cite{anderson2013steering,bornfeld2017gamifying,mutter2014behavioral}. 
At its peak, $\beta_2$ suggests that the strong-steerers will increase the number of actions (over their preferred distribution) by approximately $15$ question-votes. 

A novel insight of our model is that $\beta$ decreases below $0$ after the badge has been achieved. That is, users may decrease their activity beyond their preferred distribution after they have achieved the badge. 
This result suggests that for those users who are steered strongly, they may stop contributing altogether once the badge has been achieved.

\begin{figure}[t]
  \centering
  \includegraphics[width=.9\linewidth]{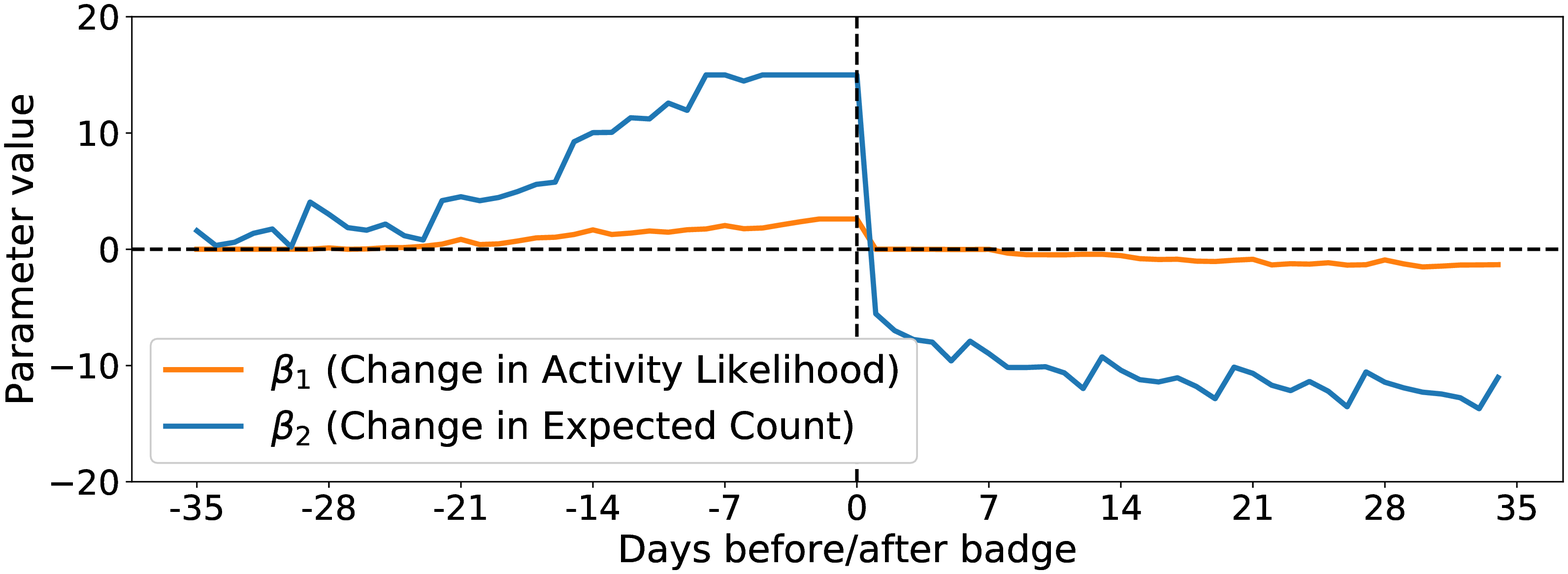}
  \caption{Expected deviation from a user's preferred distribution $P_u$.}
%   \Description{Steering beta response}
  \label{fig:steer_beta_response}
\end{figure}

Fig.~\ref{fig:steer_mean_response} presents the mean number of interactions per user as a function of the number of days until/after the badge is achieved. 
The three lines correspond to the three groups from Fig.~\ref{fig:teaser}: non-steerers, strong-steerer and the other users who are in neither of the two modes. 
We choose a deliberately low and generous cutoff $(S_{u,1}, S_{u,2}) > (.3,.3)$ to define a user as a member of the strong-steerer group and choose very conservative values $(S_{u,1}, S_{u,2}) < (.2,.1)$ to be considered as non-steerers. 
Thus it is possible that the true number of non-steerer users in the dataset might then be higher than the quoted $41.8\%$ in this analysis.
The strong-steerers make up $20.9\%$ of users and experience steering as is described by previous work --- they increase their rate of work dramatically before the badge is achieved.
Notice that the mean interaction count from these users' drops passed the other groups to close to $0$ after they achieve the badge.
% In Fig.~\ref{fig:steer_mean_response_after_badge}, we present the mean number of interactions per user for only the $5$ weeks after the badge has been achieved, such that this behavior can be properly seen.

\begin{figure}[t]
    \centering
    
    % \begin{subfigure}{\linewidth}
		\centering
		\includegraphics[width=.9\linewidth]{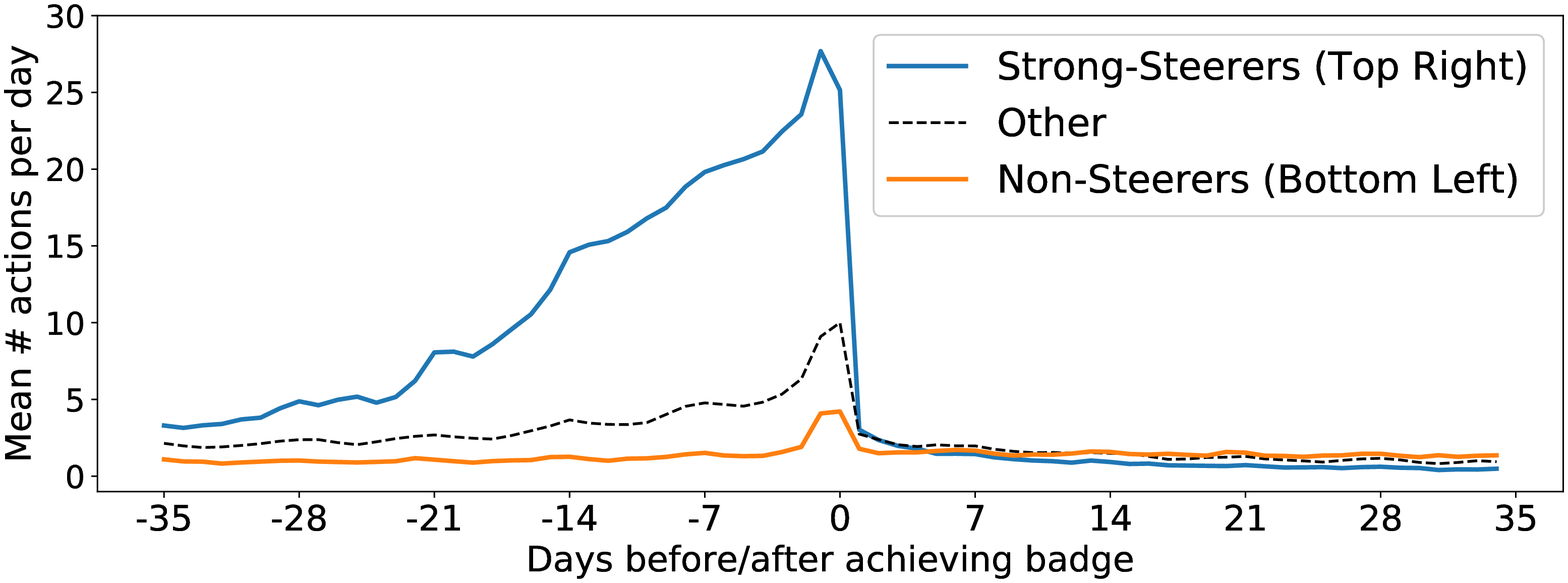}
% 		\caption{$5$ weeks before and after badge achievement.}
% 		\label{fig:steer_mean_response_all}
% 	\end{subfigure}
	
% 	\begin{subfigure}{\linewidth}
% 		\centering
% 		\includegraphics[width=\linewidth]{./images/steering_mean_response_after_badge.eps}
% 		\caption{Focused plot on only the $5$ weeks after badge achievement.}
% 		\label{fig:steer_mean_response_after_badge}
% 	\end{subfigure}

%   \includegraphics[width=\linewidth]{images/steering_mean_response.png}
  \caption{Mean number of actions per day for users who are grouped by the magnitude of their steering strength parameters ($S_u$)}
%   \Description{Steering mean response}
  \label{fig:steer_mean_response}
\end{figure}

The non-steered population (41.8\%) show no change in their interaction rates before or after the receipt of the badge. 
There is a distinct uptick in the mean number of question-vote actions on the day before and on the day of the badge achievement (Fig.~\ref{fig:steer_mean_response}, orange line). 
It is possible that this ``bump'' might mistakenly be seen as the response of the users to the badge incentive.
In fact, this bump is an artifact of the analysis technique which centers trajectories around a threshold that is crossed by the cumulative sum of the trajectory entries (see Section~\ref{sec:studying_mean_action_count} and Appendix~\ref{sec:bump_proof} for a discussion and proof of this claim). 

Table~\ref{tab:counts_per_group} presents the sizes of these three groups (when considering the entire data set). We highlight the fact that the non-steerer population is twice as large as the strong-steerer population and while the strong-steerer population is the minority, it is the highly engaged interaction activity from these users that may have led to some previous conclusions about steering.

\begin{table}
  \caption{Number of users where the latent parameter $S$ is used to threshold the amount of steering.}
  \label{tab:counts_per_group}
  \begin{center}
    \begin{tabular}{ll|c}
        \hline
        Group & ($S_1 ; S_2$) & \# Users\\
        \hline
        Non-steerer       & $(S_1,S_2) < (0.2,0.1)$   & $2890(41.79\%)$  \\
        Other             & Other                       & $2579(37.29\%)$  \\
        Strong-steerer    & $(S_1,S_2) > (0.3,0.3)$   & $1447(20.92\%)$  \\
       \hline
    \end{tabular}
   \end{center}
\end{table}

%\nh{Don't think we need this section on number of badges achieved}Our final result studies the number of badges that are achieved by the 20.9\% of the population who are characterised as strong-steerers in comparison to the 41.8\% of the population who, we claim, do not act in a manner that indicates steering. Fig.~\ref{fig:badge_counts} shows a mode on $1$ badge (given the data during the available interval) for the non-steered population with an exponential decay in the number of people who achieve two or more badges. In contrast, the steered population has a much greater count of two or more badges. We argue this is further evidence for the claim that the group in the top right of the scatter plot in Fig.~\ref{fig:teaser} actively pursues badges while the same cannot be said for the group in the bottom left. 

% \begin{figure}[hb]
%   \centering
%   \includegraphics[width=\linewidth]{./images/badge_counts.eps}
%   \caption{Normalized histogram of the number of badges achieved by the strong-steerer and non-steerer populations respectively.}
%   \Description{Steering badge counts}
%   \label{fig:badge_counts}
% \end{figure}

\section{Proving the Existence of Phantom Steering}
\label{sec:studying_mean_action_count}
The population of non-steerers in Fig.~\ref{fig:steer_mean_response} displays a sharp uptick in the mean of their action counts on the day before and on the day of the badge achievement. 
We prove that such a bump arises as an artifact of centering the data on day $0$, and it is therefore expected to arise even in the absence of a steering effect.
We show this ``phantom steering'' bump occurs in the setting of Model $0$ (Fig.~\ref{fig:graphical_model}) where daily action counts are independent draws from some unchanging latent distribution.
Our proof (and the intuition arising from it) suggests that a similar bump arises in the presence of steering as well.
It is possible that this bump may have served to inflate previous conclusions about how users change their behavior when working to achieve badges~\cite{anderson2013steering,yanovsky2019one,mutter2014behavioral}. 

For users acting under Model $0$ we present Theorem~\ref{thm:bumptheorem}, which implies that for sufficiently large badge thresholds the expected number of actions on day $0$ (the day of badge achievement) is greater than the expected number of actions on any other day.

We introduce this theorem via the following intuitive example:
Suppose that the badge threshold $N$ is chosen randomly from some large range $N \in [m,M]$ of possible action counts.
Let $S_n$ be the cumulative number of actions from a user up to (and including) day $n$.
As long as the user continues to act on the platform, $S_n$ will eventually traverse the interval $[m,M]$.
Moreover, as the count of actions on any day $n$ is a random variable (drawn from the user's preferred distribution), $S_n$ is more likely to cross the threshold $N$ on a day on which the user makes relatively more contributions. 
This claim holds even when actions are drawn under the no-steering assumptions of Model $0$ which assumes that users' action counts on each day are independent draws from their preferred distribution $P_u$ (which is not influenced by steering).

We formalize this intuition in Theorem~\ref{thm:bumptheorem}, the proof of which appears in Appendix~\ref{sec:bump_proof}. 
Recall that the random variable $A_u^0$ describes the number of actions that user $u$ performs on the day that they receive the badge.
Denote the number of actions required to achieve the badge by $N$, and let $A_{u,N}^0$ denote this random variable when the badge threshold is $N$ actions and user $u$ acts according to Model $0$. 
\begin{theorem} \label{thm:bumptheorem}
	If $P_u$ is bounded then:
	\[
		\lim_{N \rightarrow \infty}	\mathbb{E}[A_{u,N}^0] = \mathbb{E}[P_u] + \frac{Var[P_u]}{\mathbb{E}[P_u]}.
	\]
\end{theorem}
\noindent This expected bump size holds in the limit as the badge threshold becomes large with respect to the mean of $P_u$.
For fixed $P_u$ the convergence to this limit is exponential in the threshold.

Theorem~\ref{thm:bumptheorem} applies when a user's distribution is identical for all days and their actions are drawn independently from it.
In order to distinguish between days of the week as in Section~\ref{sec:inference}, we now generalize this theorem to the setting where each user has a distinct distribution for each day of the week and they draw from these distributions independently in turn.
Let $\tau \in \{1,\ldots, 7\}$ index days of the week. 
Because the $P_u^d$ are indexed by day where $d=0$ is the day of badge achievement, the sequence $P_u^1,\ldots, P_u^\tau, \ldots, P_u^7,P_u^1,\ldots$ is the vector of a user's preferred distribution for the entire period under study.%
    \footnote{Note that days are indexed arbitrarily for different users indicating that they may start their interaction, and achieve the badge, on any random day.}
We generalize Theorem~\ref{thm:bumptheorem} to the following result, the proof of which is also relegated to Appendix~\ref{sec:bump_proof}:

\begin{theorem} \label{thm:weeklybumptheorem}
	If each of the distributions $P_u^1 \ldots, P_u^\tau, \ldots, P_u^T$ is finite and nonzero (and nonnegative and integer valued), then
	\begin{equation*}
	\lim_{N\rightarrow\infty} \E[A_{u,N}^0] = \frac{\sum_{\tau=1}^T \E[(P_u^\tau)^2]}{\sum_{\tau=1}^T \E[P_u^\tau]}.
	\label{eq:expected_bump_size}
	\end{equation*}
\end{theorem}

\section{User Survey}
\label{sec:empirical_survey}
As an additional form of validation for the analytical results that are presented in this paper, we hosted a survey that recruited participants from SO to answer questions relating to their motivations for contributing to the website. 
A clickable advertisement was placed on SO and willing participants were directed to the survey. 
We paid each survey participant $\$10$ for completing the survey. In total, we received $70$ responses from the community. 
% We rejected $16$ responses as the account IDs that were associated with these users did not exist or the users had not contributed to SO, making them not part of the target population.

Fig.~\ref{fig:survey_responses_motivation} summarizes the responses to the question: 
``What are your reasons for participating in SO?" The majority of users claimed to have personal and/or altruistic reasons for contribution to the website with $87.1\%$ claiming to contribute to increase their own knowledge (and $68.6\%$ claiming to want to ``contribute to the community''). In contrast to this, only $24.2\%$ of the users selected the reason to ``achieve badges''. $50\%$ of users claimed that they had a goal of increasing their reputation score.

\begin{figure}[ht]
  \centering
  \includegraphics[width=.85\linewidth]{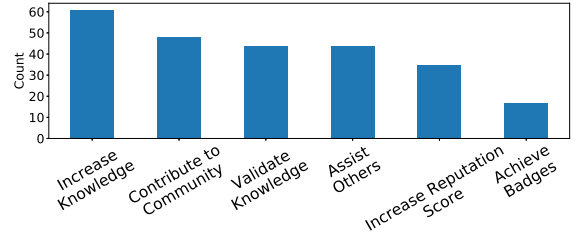}
  \caption{Counts of responses to the reasons for contributing to Stack Overflow.}
  \label{fig:survey_responses_motivation}
\end{figure}

We also asked the users specifically about their voting contributions: ``What motivates you to vote on other people's posts?''
The responses to this question are summarised in Fig.~\ref{fig:survey_responses_voting}. 
Participants could select any combination of three different reasons for voting: 
badge acquisition (``I wished to achieve one of the voting badges: e.g., Supporter, Critic, Suffrage, Vox Populi, Civic Duty or Electorate''); altruism (``I think it is important to provide feedback about other's work''), or ``other". 
Only $12.9\%$ of participants who engage in voting actions   reported badge acquisition is a motivating factor for their work. (Eight of the participants in the study claimed to not engage in voting actions and were not counted.)
%The denominator here is $62$ and not $68$ as the participants were first asked \textit{how} they contribute and were only shown the voting question if they responded that they do indeed vote on the platform.
\begin{figure}[ht]
  \centering
  \includegraphics[width=.85\linewidth]{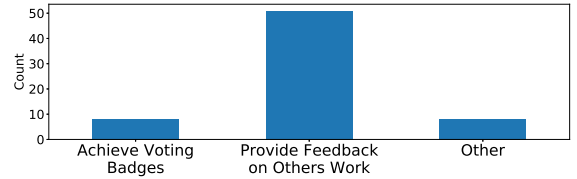}
  \caption{Counts of responses to the reasons for voting in Stack Overflow.}
%   \Description{survey responses}
  \label{fig:survey_responses_voting}
\end{figure}

Together these results present further evidence to corroborate the model predictions that only a minority of the SO participants are indeed steered by badges.

\section{Limitations}
\label{sec:limitations}

The empirical study in Section~\ref{sec:emp_study} has a number of limitations which we list here. 
However, we note that most of these limitations can be addressed in future work which is discussed in Section~\ref{sec:conclusions_and_future_work}.
% We list a number of limitations in our approach.
%, and discuss 
%future work for addressing these limitations. 
%However, we believe each of those limitations links to a possibility for future work and thus we describe these possibilities along with the limitations.

% Not Civic Duty nor Electorate: 12382519 users; 44953312 votes
% Civic Duty: 80153 users ; 51061388 votes
% Electorate: 22484 users ; 59865726 votes

First, our empirical analysis was limited to $6,916$ SO users who achieved the Electorate badge from January 2017 to April 2019.
The Electorate badge is the highest badge category for vote actions and while the number of users who achieve this badge is a small fraction of all users who vote on the platform ($0.18\%$), these users account for $38.4\%$ of all votes that have been performed on SO.
In Appendix~\ref{sec_appendix:civic_duty} we have included a qualitatively similar study on the larger population of $25,314$ users who achieved the Civic Duty badge. The users who achieve this badge represent $0.64\%$ of all users who vote on SO and they account for $32.8\%$ of all votes cast. The last group ($99.2\%$ of all users who vote) accounts for only $28.8\%$ of all votes cast and have achieved neither the Civic Duty nor the Electorate badge.
We also motivate that studying the voting actions helps to provide insight into badge effect on user behaviour as voting actions do not receive reputation points on SO and therefore they do not have this as a possible confounding variable (as would be present for the editing badges). 
See Appendix~\ref{sec_appendix:civic_duty} for an extended analysis from the freely available editing data. 

A second limitation that is related to the first is our focus on only one action type, that of voting-actions. 
We chose to do this in order to study the effect of the badge on the action that it directly incentivizes, which is presumably associated with the largest response to the badge ($\beta$ in our model).
This is a limitation as previous work has raised the possibility that a badge of one type (e.g., incentivizing voting actions) can have effects on actions of another type (e.g., editing or reviewing actions)~\cite{anderson2013steering,mutter2014behavioral}. 
We note this is a limitation of our empirical study and not a limitation of the models that have been presented as the models can extend easily to this new setting where many actions are included in a new likelihood model.

A third limitation is that we explicitly assume a single parametric form for the steering parameter $\beta$.
This assumption is that all steered users deviate from their preferred distribution in the same way, and that users only differ in the strength parameter $S$.
However, it is clear that a significant portion of the population ($\pm 40\%$) are neither steered nor non-steered (they are not in the two modes that are described in Fig.~\ref{fig:teaser}) leading to the conclusion that this simplifying assumption is inadequate to capture the population of users at large. 

% Finally, we note a limitation of the modeling approach outlined in Section~\ref{sec:modeling_user_contrbutions}. While the model does provide a good fit to the data, it does not reason about cognitive aspects affecting users' behavior, such as intention to achieve a badge or motivation to contribute to the site.
% This is natural given the observational form of this study and we rather focus on providing descriptive insights into peoples' behavior.
% The cognitive processes and intrinsic motivations for pursuing badges remains unknown.

\section{Conclusion and Future work}
\label{sec:conclusions_and_future_work}

We have presented a novel probabilistic model that describes how users interact on the SO platform and in particular how these users respond to badge incentives on the website. 
We demonstrated how this model can be fit to the data that is provided by SO and we investigated the distribution that is learnt over the latent space that describes the ``steering effect''. 

Our results  provide a more informed  understanding of how users respond to badges in online communities. 
First, that some users do exhibit steering supports the claims made by previous work~\cite{anderson2013steering,mutter2014behavioral,yanovsky2019one,li2012quantifying}.
However, approximately 40\% of the population do not exhibit steering. 
These users do not change the rate of their activity for the $10$ weeks under study. Rather, they continue to act with the same rate well after the badge has been achieved. 
This suggests that these users have reasons for performing voting actions on SO which do not include the desire to obtain the Electorate badge.

Second, the 20\% of the population identified as strong-steerers significantly decrease their level of contributions after day zero, beyond what was previously reported. 
It is possible that assigning additional badges, with thresholds beyond those already in place in SO will continue to motivate such users.

Third, any analysis of badge behavior must take into account the presence of the phantom steering bump which has not previously been acknowledged in the context of badges. 
This statistical artifact is model independent and  may lead to inflated conclusions about the effect of badges on users' behavior.

Future work will extend the models of Section~\ref{sec:modeling_user_contrbutions} to allow for different responses $\beta$ to a badge. 
We will update the model to include \textit{multiple $\beta$} parameters to capture different responses to badges across users. 
It is possible that the roughly $40\%$ of users who are not well described by Model $2$ have responses that are not described by a single  $\beta$ parameter and allowing for this more intricate model will allow it to capture more users into some behavioral archetype.
Moreover, we wish to study the {indirect} effect of badges of a certain type on actions that are not directly incentivized by this badge, as stated by \cite{anderson2013steering} and \cite{li2012quantifying}. 
Our model can be directly applied to studying this question by modeling the likelihood of activities as a vector for each action type. 

% A final exciting note for this work is that we are in the process of working with SO to run a survey of the users on the website. 
A deeper understanding of how and why users contribute to these peer-production websites will inform the design of  more personalized and effective rewards that motivate and engage the users.
Our user study has highlighted that reputation points and altruistic factors such as ``impact'' may contribute more than badges towards motivating users on SO. 
We therefore propose to design models that assist in understanding how these different rewards can be leveraged to incentivize users on such platforms.

\section*{Acknowledgment}
The authors would like to thank Stack Overflow for providing us with the voting data that allowed for this investigation. Moreover, Stack Overflow supplied advertisement space on its website such that a link to our survey was visible to their userbase. 

Hoernle is funded by a Commonwealth Scholarship. Kehne and Procaccia were partially supported by the National Science Foundation under grants CCF-2007080, IIS-2024287 and CCF-1733556; and by the Office of Naval Research under grant N00014-20-1-2488.

\bibliographystyle{IEEEtran}
\bibliography{bibliography}

% Generated by IEEEtran.bst, version: 1.14 (2015/08/26)
\begin{thebibliography}{10}
\providecommand{\url}[1]{#1}
\csname url@samestyle\endcsname
\providecommand{\newblock}{\relax}
\providecommand{\bibinfo}[2]{#2}
\providecommand{\BIBentrySTDinterwordspacing}{\spaceskip=0pt\relax}
\providecommand{\BIBentryALTinterwordstretchfactor}{4}
\providecommand{\BIBentryALTinterwordspacing}{\spaceskip=\fontdimen2\font plus
\BIBentryALTinterwordstretchfactor\fontdimen3\font minus
  \fontdimen4\font\relax}
\providecommand{\BIBforeignlanguage}[2]{{%
\expandafter\ifx\csname l@#1\endcsname\relax
\typeout{** WARNING: IEEEtran.bst: No hyphenation pattern has been}%
\typeout{** loaded for the language `#1'. Using the pattern for}%
\typeout{** the default language instead.}%
\else
\language=\csname l@#1\endcsname
\fi
#2}}
\providecommand{\BIBdecl}{\relax}
\BIBdecl

\bibitem{hull1932goal}
C.~L. Hull, ``The goal-gradient hypothesis and maze learning.''
  \emph{Psychological Review}, vol.~39, no.~1, p.~25, 1932.

\bibitem{kivetz2006goal}
R.~Kivetz, O.~Urminsky, and Y.~Zheng, ``The goal-gradient hypothesis
  resurrected: Purchase acceleration, illusionary goal progress, and customer
  retention,'' \emph{Journal of Marketing Research}, vol.~43, no.~1, pp.
  39--58, 2006.

\bibitem{mutter2014behavioral}
T.~Mutter and D.~Kundisch, ``Behavioral mechanisms prompted by badges: The
  goal-gradient hypothesis,'' in \emph{International Conference on Information
  Systems}, 2014.

\bibitem{anderson2013steering}
A.~Anderson, D.~Huttenlocher, J.~Kleinberg, and J.~Leskovec, ``Steering user
  behavior with badges,'' in \emph{Proceedings of the 22nd international
  conference on World Wide Web}.\hskip 1em plus 0.5em minus 0.4em\relax ACM,
  2013, pp. 95--106.

\bibitem{yanovsky2019one}
S.~Yanovsky, N.~Hoernle, O.~Lev, and K.~Gal, ``One size does not fit all: Badge
  behavior in q\&a sites,'' in \emph{Proceedings of the 27th ACM Conference on
  User Modeling, Adaptation and Personalization}.\hskip 1em plus 0.5em minus
  0.4em\relax ACM, 2019, pp. 113--120.

\bibitem{kusmierczyk2018causal}
T.~Kusmierczyk and M.~Gomez-Rodriguez, ``On the causal effect of badges,'' in
  \emph{Proceedings of the 2018 World Wide Web Conference}, 2018, pp. 659--668.

\bibitem{bornfeld2017gamifying}
B.~Bornfeld and S.~Rafaeli, ``Gamifying with badges: A big data natural
  experiment on stack exchange,'' \emph{First Monday}, vol.~22, no.~6, 2017.

\bibitem{ipeirotis2014quizz}
P.~G. Ipeirotis and E.~Gabrilovich, ``Quizz: targeted crowdsourcing with a
  billion (potential) users,'' in \emph{Proceedings of the 23rd international
  conference on World Wide Web}.\hskip 1em plus 0.5em minus 0.4em\relax ACM,
  2014, pp. 143--154.

\bibitem{anderson2014engaging}
A.~Anderson, D.~Huttenlocher, J.~Kleinberg, and J.~Leskovec, ``Engaging with
  massive online courses,'' in \emph{Proceedings of the 23rd international
  conference on World Wide Web}.\hskip 1em plus 0.5em minus 0.4em\relax ACM,
  2014, pp. 687--698.

\bibitem{zhang2019reading}
H.~Zhang, S.~Wang, T.-H. Chen, and A.~E. Hassan, ``Reading answers on stack
  overflow: Not enough!'' \emph{IEEE Transactions on Software Engineering},
  2019.

\bibitem{blei2014build}
D.~M. Blei, ``Build, compute, critique, repeat: Data analysis with latent
  variable models,'' \emph{Annual Review of Statistics and Its Application},
  vol.~1, pp. 203--232, 2014.

\bibitem{rezende2015variational}
D.~J. Rezende and S.~Mohamed, ``Variational inference with normalizing flows,''
  \emph{arXiv preprint arXiv:1505.05770}, 2015.

\bibitem{kingma2013auto}
D.~P. Kingma and M.~Welling, ``Auto-encoding variational bayes,'' \emph{arXiv
  preprint arXiv:1312.6114}, 2013.

\bibitem{kingma2016improved}
D.~P. Kingma, T.~Salimans, R.~Jozefowicz, X.~Chen, I.~Sutskever, and
  M.~Welling, ``Improved variational inference with inverse autoregressive
  flow,'' in \emph{Advances in Neural Information Processing Systems}, 2016,
  pp. 4743--4751.

\bibitem{ranganath2014black}
R.~Ranganath, S.~Gerrish, and D.~Blei, ``Black box variational inference,'' in
  \emph{Artificial Intelligence and Statistics}, 2014, pp. 814--822.

\bibitem{box1962useful}
G.~E. Box and W.~G. Hunter, ``A useful method for model-building,''
  \emph{Technometrics}, vol.~4, no.~3, pp. 301--318, 1962.

\bibitem{papamakarios2019normalizing}
G.~Papamakarios, E.~Nalisnick, D.~J. Rezende, S.~Mohamed, and
  B.~Lakshminarayanan, ``Normalizing flows for probabilistic modeling and
  inference,'' \emph{arXiv preprint arXiv:1912.02762}, 2019.

\bibitem{neal1993probabilistic}
R.~M. Neal, \emph{Probabilistic inference using Markov chain Monte Carlo
  methods}.\hskip 1em plus 0.5em minus 0.4em\relax Department of Computer
  Science, University of Toronto Toronto, ON, Canada, 1993.

\bibitem{blei2003latent}
D.~M. Blei, A.~Y. Ng, and M.~I. Jordan, ``Latent dirichlet allocation,''
  \emph{Journal of Machine Learning Research}, vol.~3, no. Jan, pp. 993--1022,
  2003.

\bibitem{hoffman2013stochastic}
M.~D. Hoffman, D.~M. Blei, C.~Wang, and J.~Paisley, ``Stochastic variational
  inference,'' \emph{The Journal of Machine Learning Research}, vol.~14, no.~1,
  pp. 1303--1347, 2013.

\bibitem{li2012quantifying}
Z.~Li, K.-W. Huang, and H.~Cavusoglu, ``Quantifying the impact of badges on
  user engagement in online q\&a communities,'' in \emph{International
  Conference on Information Systems}, 2012.

\end{thebibliography}

\appendix

\subsection{Omitted Proofs}
\label{sec:bump_proof}

Here we present the proof of Theorem~\ref{thm:bumptheorem}.
Let $X$ be a nonnegative, bounded, and integer-valued random variable.
Let $\{X_m\}_{m \in \N}$ be independent random variables which are distributed identically to $X$.
We will be concerned with the partial sums $S_n = \sum_{m=1}^n X_m$.
Let $Y_N$ denote the random variable which is the copy $X^m$ that brings $S_n$ across the threshold $N$; that is, for which $S_{m-1} < N$ and $S_m \geq N$. 

\begin{theorem} \label{thm:appendixsimple}
	If $X$ is nonnegative, integer-valued, and bounded then
	\[
		\lim_{N \rightarrow \infty}	\E[Y_N] = \E[X] + \frac{Var[X]}{\E[X]}
	\]
\end{theorem}

More generally, we also consider the case when the $X$ are drawn from distributions $X^1, \ldots, X^\tau, \ldots, X^T$ repeatedly in turn.
Then the partial sums are ${S_n = \sum_{m=1}^n X^{m\mod T}}$, where all copies of $X^\tau$ are independent.
Let $\xi_\tau$ denote the event that $Y_N$ is drawn from distribution $X^\tau$, and let $Z = \sum_{\tau = 1}^D X^\tau$.
For this setting we have the following theorem:
\begin{theorem} \label{thm:appendixweekly}
If each of the distributions $X^\tau$ is finite, nonzero, nonnegative, and integer valued then
\[
\lim_{N\rightarrow\infty} \E[Y_N] = \frac{\sum_\tau \E[(X^\tau)^2]}{\E[Z]}.
\]
\end{theorem}

Theorem~\ref{thm:appendixsimple} follows directly from Theorem~\ref{thm:appendixweekly} by taking the $X^\tau$ to be identically distributed. It therefore suffices to prove Theorem~\ref{thm:appendixweekly}.

We begin by showing that the likelihood of the sequence $\{S_n\}$ visiting any given number $N$ is asymptotically uniform. 
Let 
$p_m  := \E\left[ \left|\left\{ n \in \N: S_n = m \right\} \right| \right]$, 
$g := \gcd(range(X))$
and observe that if $X > 0$ then $p_m = \Pr[ m \in \{S_n\}]$.
Also, if $m \not \in g\N$ then clearly $p_m = 0$.
For the $p_m$ for which $m \in g \N$, we have the following lemma:
\begin{lemma} \label{asymptoticallyuniform}
	If $X$ is nonzero, nonnegative, and bounded then
	\begin{equation*}
		\lim_{n \rightarrow \infty} p_{gn} = \frac{g}{\mu}
	\end{equation*}
\end{lemma}

\begin{proof} 
    First, it suffices to assume that $g = 1$.
	This is because the integer-valued random variable $X' := X/g$ has mean $ \mu / g$ and $\gcd(range(X')) = 1$, and proving the claim for $X'$ implies the claim for $X$. 
	It also suffices to assume that $X > 0$. 
	This is because the sequence $\{S_n\}_{n\in\N}$ remains at a specific value $m$ only so long as the independent draws are $X_n = 0$, after which it leaves $m$ forever.
	The expected number of steps that $\{S_n\}$ lingers at $m$ for is exactly $\frac{1}{1-\alpha}$, where $\alpha = \Pr[X = 0]$. 
	Since $\mu > 0$ by assumption, we may prove the claim for $X'' := X \vert X>0$. Then $\mu = \frac{\mu''}{1-\alpha}$ and 
	\begin{equation*}
		p_m = \E\left[ \left|\left\{ n \in \N: S_n = m \right\} \right| \right]  = \frac{1}{1-\alpha} p_m''
	\end{equation*}
	Thus proving the claim for $X''$, proves the claim for $X$. Therefore, we may assume without loss of generality that $X > 0$ and that $\gcd(range(X)) = 1$. 
	
	Let ${M := \max\{range(X)\}}$ be the maximum value that $X$ obtains.
	Then the $p_m$ obey the recurrence 
	\begin{equation} \label{originalrecurrence}
		p_m = \sum_{j=1}^M p_{m-j} \Pr[X = j]
	\end{equation}
	with the initial conditions $p_0 = 1$ and $p_m = 0$ for all $m < 0$.
	Because $X$ is bounded by $M$, we may break $\N$ up into ``epochs'' ${\{1, \ldots, M\}, \{M+1, \ldots, 2M\}, \ldots}$, and then define ${q^k_r := p_{kM + r}}$ with ${q^0 := (0, \ldots, 0, 1)^T}$.
	For any $m = kM + r$ we can then iteratively expand the $p_{m-j}$ terms in Equation~\ref{originalrecurrence} for which $m-j \geq k M$ until the expression for each $p_m$ depends only on the previous epoch, which gives an alternative recurrence of the form
	\begin{equation} \label{newrecurrence}
		p_{kM + r} = \sum_{s=1}^M \alpha^r_s \: p_{(k-1)M+s}
	\end{equation}
	where $r, s \in [M]$ (and the initial conditions are the values of $p_s$ for $s\in [M]$).
	Note that these $\alpha^r_s$ do not depend on $k$. 
	The recurrences in Equations \ref{originalrecurrence} and \ref{newrecurrence} give $p_m$ as a convex combination of previous values, and so we may rewrite Equation~\ref{newrecurrence} as $q^k = A^k q^0$, where $A := \{\alpha^r_s\}_{r,s \in [M]}$ is a right stochastic square matrix.
	Furthermore it follows from the assumption $g=1$ that $A$ is primitive.
	Therefore the Perron-Frobenius Theorem implies that $A^k$ converges exponentially quickly to a matrix of the form $\vec{1}\vec{u}^T$, where $\vec{1}$ and $\vec{u}^T$ are the unique right and left eigenvectors of $A$ corresponding to the eigenvalue $\lambda = 1$. 
	This in turn implies that $q^k = A^k q^0$ converges to some uniform vector $(\gamma, \ldots, \gamma)$, and therefore that $\lim_{m \rightarrow \infty} p_m =\gamma$.

	Finally we argue that $\gamma = 1/\mu$. 
	We can show this by considering $C(N,J) := \E\left[\left| \{S_n\} \cap [N,J)\right|\right]$, 
	the mean number of times that $\{S_n\}$ intersects some interval $[N,J)$. 
	Since the $p_m$ converge, for fixed $J$ we may use linearity of expectation to choose $N$ large enough to guarantee that $C(N,J) \in J\gamma \pm \epsilon$ for any given $\epsilon > 0$. 
	On the other hand, by considering the $\{S_n\}$ as ``restarting'' when they reach the epoch preceding $N$, we may use the central limit theorem to argue that $C(N,J) \in \frac{J}{\mu} \pm O(J^{2/3})$. 
	Taking the limit as $J$ becomes large yields $\gamma = 1/\mu$. 
\end{proof}

% With this lemma in hand, we are ready to analyze the distribution of $Y_N$ in the limit. 
\begin{proof} [Proof of Theorem~\ref{thm:appendixweekly}]
%We show this by adapting the proof of Theorem~\ref{thm:bumptheorem}.
First, we may assume without loss of generality that the gcd of the supports of $X^1 \ldots X^T$ is $g=1$. 
To see this, define new integer random variables $\hat X^1 = X^1/g, \ldots, \hat X^T=X^T/g$.
If the claim holds for $\hat X^1, \ldots, \hat X^T$ then since $Y_{gN} = g \hat Y_{N}$ it follows that:
\begin{equation*}
\lim_{N\rightarrow\infty} \E[Y_N] = g \lim_{N' \rightarrow\infty} \E[\hat Y_{N'}] = \frac{\sum_\tau \E[g^2 (\hat X^\tau)^2]}{\E[g \hat Z]} = \frac{\sum_\tau \E[(X^\tau)^2]}{\E[Z]}
\end{equation*}

\noindent For a fixed threshold $N$, we are interested in the event that both $Y_N$ is a copy of $X^\tau$ and that $Y_N = k$. 
Then 
	\begin{align*}
	\Pr[\xi_\tau &\:\wedge\: Y_N = k] \\
	&= \Pr \left[ \bigvee_{n=1}^\infty \left(\xi_\tau \:\wedge\: X_n = k   \:\wedge\:  S_{n-1} \in [N-k, N) \right)  \right] \\
	&= \Pr \left[ \bigvee_{n=1}^\infty \left(X^\tau = k   \:\wedge\:  S_{nD + \tau - 1} \in [N-k, N) \right)  \right] \\
	&= \sum_{n=1}^\infty \Pr\left[X^\tau = k \:\wedge\: S_{nD + \tau -1} \in [N-k, N) \right]. \\
	&= \Pr[X^\tau = k] \sum_{n=1}^\infty \Pr\left[S_{nD + \tau - 1} \in [N-k, N) \right] \\
	&= \Pr[X^\tau = k] \E\left[ \left| \left\{ n \in \N : S_{nD+\tau-1} \in [N-k, N)  \right\} \right| \right]\\
	&= \Pr[X^\tau = k] \sum_{m = N-k}^{N-1} \E\left[ \left| \left\{ n \in \N : S_{nD+\tau-1} = m  \right\} \right| \right]
	\end{align*}

It follows that
\begin{align*}
	\lim_{N \rightarrow \infty}	\Pr[\xi_\tau & \:\wedge\: Y_N = k] \\
	&=  \Pr[X^\tau = k] \lim_{N\rightarrow \infty} \sum_{m = N-k}^{N-1} \Pr\left[ m \in \{S_{nD+\tau-1}\} \right] \\
	&=  \Pr[X^\tau = k] \frac{k}{\E[Z]},
\end{align*}
	where the last equality holds by passing the limit through the sum and applying Lemma~\ref{asymptoticallyuniform} to $Z$ in order to conclude that each of the probabilities on the right hand side approaches $1/\E[Z]$. 
	Therefore, 
% 	\ap{Is this next equation needed/used?}
% 	\begin{align*}
% 	\lim_{N \rightarrow \infty}	\Pr[\xi_\tau] &=  \sum_k \Pr[X^\tau = k] \frac{k}{\E[Z]} = \frac{\E[X^\tau]}{\E[Z]}
% 	\end{align*}
% 	and 
	\begin{align*}
	\lim_{N \rightarrow \infty} \Pr[Y_N = k] &= \sum_{\tau}\lim_{N \rightarrow \infty} \Pr[\xi_\tau \:\wedge\: Y_N = k]\\
	&=\frac{k}{\E[Z]} \sum_\tau \Pr[X^\tau = k]
	\end{align*}
	We conclude that
	\begin{align*}
	\lim_{N\rightarrow \infty} \E[Y_N] &= \frac{1}{\E[Z]} \sum_\tau \sum_k k^2 \Pr[X^\tau = k] = \frac{ \sum_\tau \E[(X^\tau)^2]}{\E[Z]}.
\end{align*}
\end{proof}

\subsection{Plots for Strunk \& White and Civic Duty Achievers}
\label{sec_appendix:civic_duty}

The following plots show the $\beta$ deviation for the Civic Duty ($25,314$ users) and the Strunk \& White  ($22,496$ users) populations and the mean plots of actions for the groups inferred from the $S_u$ parameter. Code available at: \url{https://github.com/NickHoernle/icdm_2020}.

\begin{figure}[hb]
    \centering
        \subfloat[Civic Duty\label{1a1}]{%
            \includegraphics[width=0.8\linewidth]{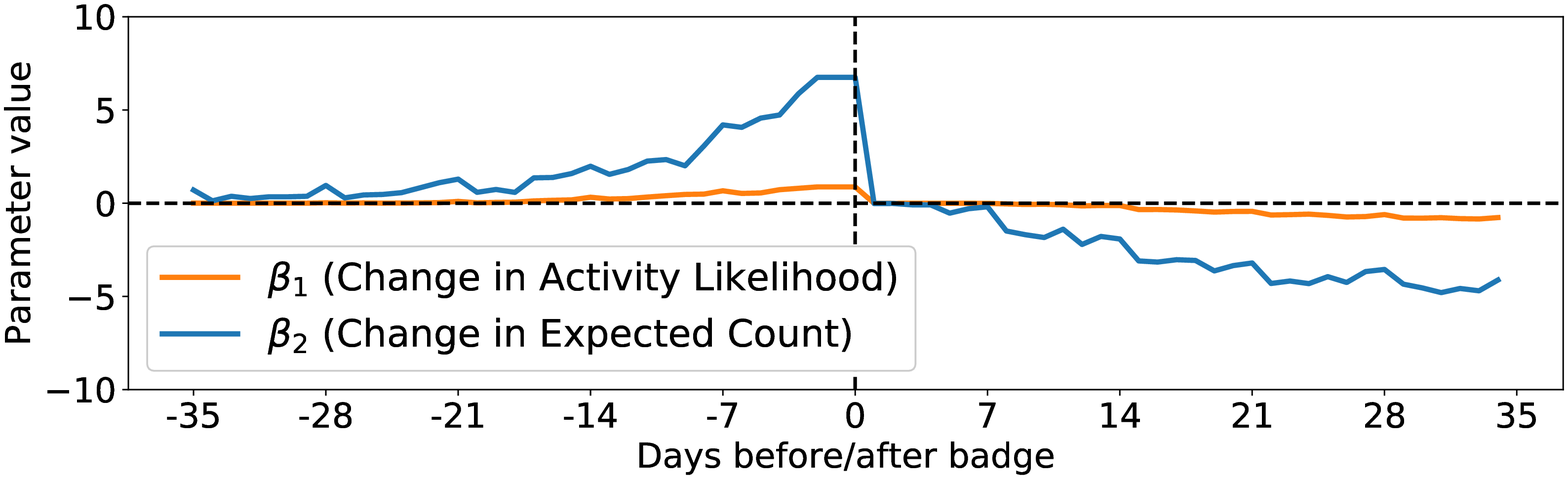}}
        \\
        \subfloat[Strunk\&White\label{1b1}]{%
            \includegraphics[width=0.8\linewidth]{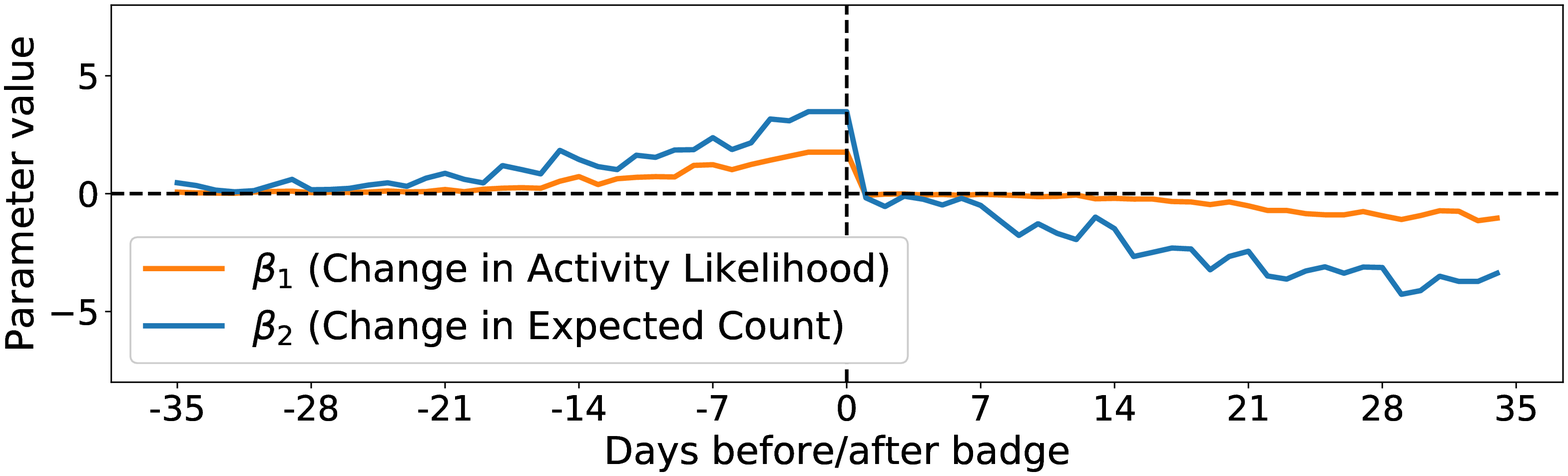}}
	\caption{Expected deviation from a user's preferred distribution $P_u$}
	\label{fig:steer_beta__civic_and_sw}
\end{figure}

\begin{figure}[hb]
    \centering
        \subfloat[Civic Duty\label{1a2}]{%
            \includegraphics[width=0.8\linewidth]{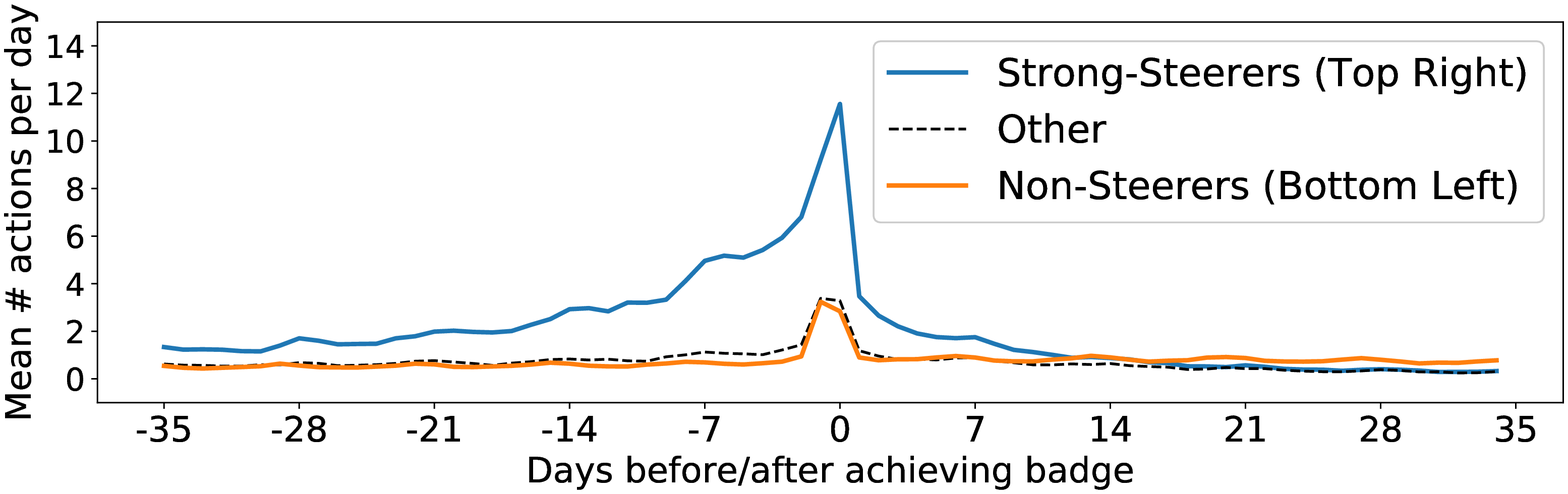}}
        \\
        \subfloat[Strunk\&White\label{1b2}]{%
            \includegraphics[width=0.8\linewidth]{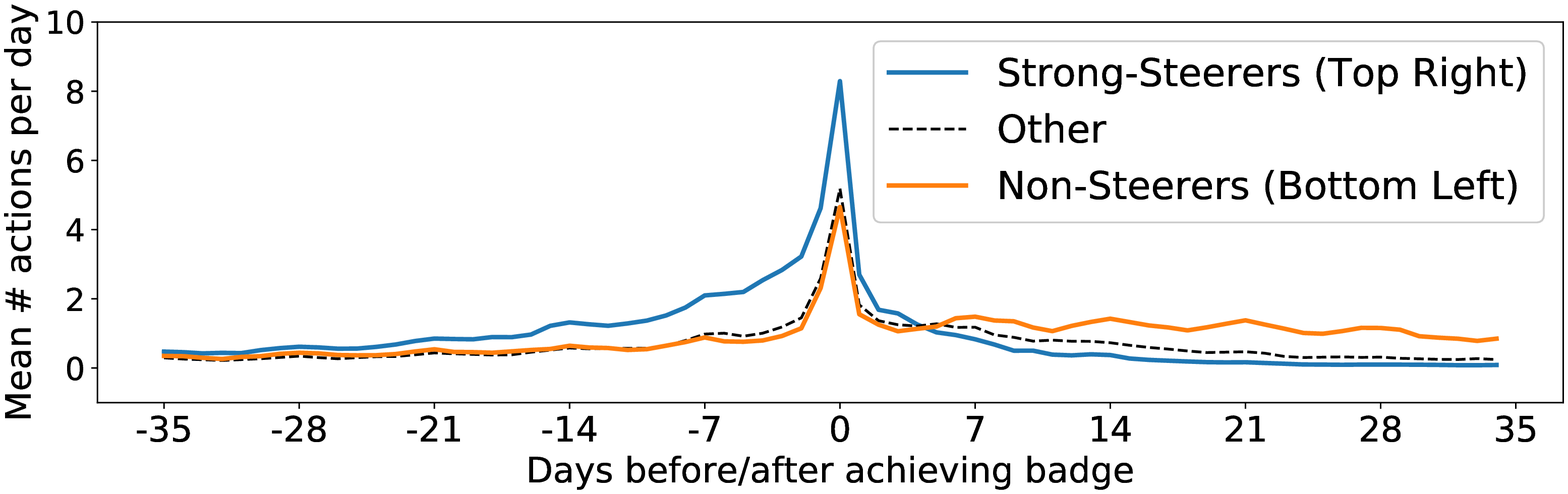}}
	\caption{Mean number of actions per day for users who are grouped by the magnitude of their steering strength parameters (S).}
	\label{fig:steer_mean_response__civic_and_sw}
\end{figure}
\end{document}